\documentclass{ecai} 

\usepackage{latexsym}
\usepackage{amssymb}
\usepackage{amsmath}
\usepackage{amsthm}
\usepackage{booktabs}
\usepackage{enumitem}
\usepackage{graphicx}
\usepackage{color}

\newtheorem{theorem}{Theorem}

\newtheorem{proposition}[theorem]{Proposition}

\newtheorem{definition}{Definition}

\newcommand{\BibTeX}{B\kern-.05em{\sc i\kern-.025em b}\kern-.08em\TeX}

\usepackage{url}

\usepackage{booktabs} 
\usepackage{amsmath}
\usepackage{graphicx}
\usepackage{leftidx}
\usepackage{subeqnarray}
\usepackage{cases}
\usepackage{textcomp}
\usepackage{subfigure}
\usepackage{epstopdf}
\usepackage{multicol}
\usepackage{array} 
\usepackage{color}
\usepackage{bm}
\usepackage{algorithm}
\usepackage{algpseudocode}
\usepackage{threeparttable}
\usepackage{xspace}
\usepackage{algorithmicx}
\usepackage{algorithm}
\usepackage{algpseudocode}
\usepackage{multirow}
\usepackage{booktabs}
\usepackage{amsmath}
\usepackage{siunitx}
\usepackage{pifont}
\usepackage{wrapfig}
\usepackage{bbm}
\usepackage{xr}
\usepackage[normalem]{ulem}
\usepackage{threeparttable}
\usepackage[table,xcdraw]{xcolor}

\begin{document}

\begin{frontmatter}

\paperid{911}

\title{Anti-Matthew FL: Bridging the Performance Gap in Federated Learning to Counteract the Matthew Effect}

\author[1]{\fnms{Jiashi}~\snm{Gao}}
\author[2]{\fnms{Xin}~\snm{Yao}}
\author[1]{\fnms{Xuetao}~\snm{Wei}\thanks{Corresponding Author. Email: weixt@sustech.edu.cn.}} 

\address[1]{Department of Computer
Science and Engineering, Southern University of Science and Technology, Shenzhen, China}
\address[2]{Department of Computing and Decision Sciences, Lingnan University, Hong Kong SAR, China}
\begin{abstract}
Federated learning (FL) stands as a paradigmatic approach that facilitates model training across heterogeneous and diverse datasets originating from various data providers.
However, conventional FLs fall short of achieving consistent performance, potentially leading to performance degradation for clients who are disadvantaged in data resources.
Influenced by the \textbf{Matthew effect}, deploying a performance-imbalanced global model in applications further impedes the generation of high-quality data from disadvantaged clients, exacerbating the disparities in data resources among clients. 
In this work, we propose \textit{anti-Matthew fairness} for the global model at the client level, requiring equal accuracy and equal decision bias across clients. To balance the trade-off between achieving  \textit{anti-Matthew fairness} and performance optimality, we formalize the  anti-Matthew effect federated learning (\textit{anti-Matthew FL})  as a multi-constrained multi-objectives optimization (MCMOO) problem and propose a three-stage multi-gradient descent algorithm to obtain the Pareto optimality. 
 We theoretically analyze the convergence and time complexity of our proposed algorithms. Additionally, through extensive experimentation,  we demonstrate that our proposed \textit{anti-Matthew FL} outperforms other state-of-the-art FL algorithms in achieving a high-performance global model while effectively bridging performance gaps among clients. We hope this work provides valuable insights into the manifestation of the \textbf{Matthew effect} in FL and other decentralized learning scenarios and can contribute to designing fairer learning mechanisms, ultimately fostering societal welfare. 
\end{abstract}
\end{frontmatter}
\section{Introduction}
\label{sec:intro}

Federated learning (FL) \cite{mcmahan2017communication} has emerged as a  significant learning paradigm in which clients utilize their local data to train a global model collaboratively without sharing data, and has attracted researchers from various fields, especially in domains where data privacy and security are critical, such as healthcare, finance, and social networks \cite{10.1145/3501296,10.1145/3298981,Long2020}. 



\subsection{Matthew Effect in FL} In the real world, clients may have unequal data resources due to historical or unavoidable social factors. They deserve fair treatment based on social welfare and equality principles. However, the existing \textit{contribution fairness} \cite{tay2022incentivizing,liu2022contribution}, which requires that the model performance on each client is proportional to their data resource contributed, worsens the resource plight of poorer clients. For instance,  when hospitals with non-i.i.d. datasets collaborate to train a disease diagnosis model, the hospitals with lower data resources will receive a model that does not fit well with their data distributions, as high-resource hospitals dominate the collaborative model training more. Therefore, the local model performance in a low-resource hospital may exhibit uncertain accuracy and decision bias. Such a low-trustworthy model may affect the subsequent diagnosis of low-resource clients, leading to persistent resource inequality and the deterioration of social welfare. This phenomenon is referred to as the \textbf{Matthew effect} \cite{merton1968matthew}, a social psychological phenomenon that describes how the rich get richer and the poor get poorer in terms of resources such as education, economy, and data information, etc.

\subsection{Anti-Matthew Fairness} 
To mitigate the model performance disparities arising from the Matthew effect, we introduce the concept of \textit{anti-Matthew fairness}, which requires  the global model to exhibit equal performance across clients. We focus on two critical performances: \textbf{accuracy and decision bias}. Accuracy reflects how effectively the global model fits the local data of clients; hence, \textit{equal accuracy performance can improve decision quality for clients with limited resources}, naturally mitigating their disadvantage and  enhancing their subsequent ability to leverage model-empowered solutions to bridge the gap with other advantaged clients. Decision bias reflects the fairness of the global model's decisions concerning different protected groups, such as gender, race, or religion, within the client. Achieving \textit{equal decision bias performance can enhance the reputation and decision credibility of the disadvantaged clients}, indirectly strengthening their ability to generate and contribute more valuable data resources.

\subsubsection{Challenges and Considerations} Attaining  \textit{anti-Matthew fairness} in FL settings confronts the following challenges: \ding{182} First,  \textit{anti-Matthew fairness} requires equal performance across  clients, introducing a trade-off with achieving the highest performance during model training. This trade-off becomes particularly evident when dealing with heterogeneous clients, where the global model exhibits varied performance. For example, an advantaged client must make a trade-off by sacrificing local performance for the sake of \textit{anti-Matthew fairness}, as exceptional high local performance might result from the model's poor performance on other clients, which is undesirable from the standpoints of social welfare and ethics. Our goal is to strike a balance in this trade-off, ensuring that all clients achieve both high and equitable performance; \ding{183} The second challenge arises from considering \textit{anti-Matthew fairness} in both accuracy and decision bias. It is noted that there exists a potential trade-off between accuracy and decision bias \cite{wang2021understanding}.

\subsubsection{Solutions} To address the challenges, we propose a novel FL  framework named \textit{anti-Matthew FL} to bridge performance gaps across clients. We first quantify  \textit{anti-Matthew fairness}, allowing the training objective to be formulated as a multi-constrained multi-objectives optimization (MCMOO) problem. As shown in Fig. \ref{fig:framework}, the objectives are to minimize the local empirical risk losses $\left \{ l_1,...,l_N \right \} $  on $N$ clients. The local decision biases $\left \{ f_1,...,f_N \right \} $  on $N$ clients are constrained to be below an acceptable bias budget. These two goals jointly ensure higher performance of model $h$ across clients. We impose constraints on the deviations of local empirical risk loss and local decision bias from their mean values to achieve \textit{anti-Matthew fairness} with equal performance among clients. We then propose a three-stage multi-gradient descent algorithm for global model training that converges to Pareto optimal solutions, in which the global model performance on each objective is maximized within the decision space and cannot be further improved without harming others.

\begin{figure}[t]
\centering
\includegraphics[width=8.5cm]{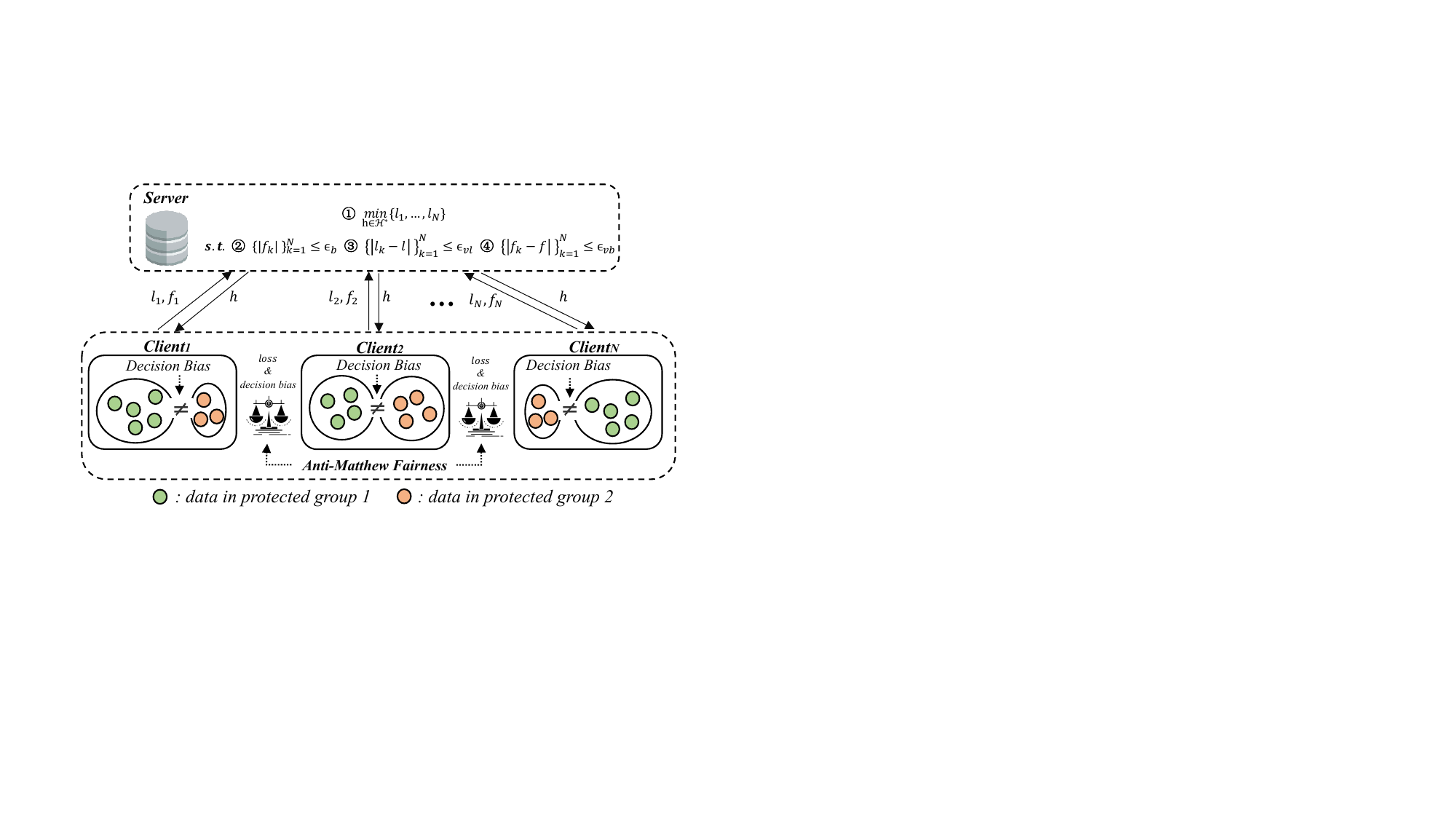}
\caption{Training goals in \textit{anti-Matthew FL}.}
\label{fig:framework}
\end{figure}
\subsection{Contribution}
In summary, this paper makes the following contributions:  \ding{182}
We propose \textit{anti-Matthew fairness}, aiming to simultaneously bridge the gaps in accuracy and bias of the global model across clients. This is particularly necessary to mitigate the Matthew effect in collaborative training and is valuable to consider for welfare-oriented collaborative training;
\ding{183} We formally define \textit{anti-Matthew federated learning } as a multi-constrained multi-objectives optimization (MCMOO) problem and propose a three-stage multi-gradient descent algorithm that achieves Pareto optimal solutions. We theoretically
analyze the convergence and time complexity of the proposed algorithms;  
\ding{184} We perform comprehensive experiments on both synthetic and real-world datasets to validate the effectiveness, convergence, hyperparameter sensitivity, rationality (via ablation study), and robustness of the proposed \textit{anti-Matthew FL}. The results show that \textit{anti-Matthew FL} outperforms state-of-the-art (SOTA) fair FL baselines in terms of overall performance and \textit{anti-Matthew fairness} across clients.

 \section{Related Work}
Most of the existing fairness research \cite{10.1145/3580305.3599180,10172501,shen2022fair,choi2021group,sankar2021matchings} assumes that the training process can access the whole training dataset in a centralized manner. However, this assumption does not hold when privacy protections prevent clients from sharing their data with a central server.
In FL, several studies \cite{ezzeldin2023fairfed,papadaki2022minimax,chu2021fedfair,du2021fairness,cui2021addressing,hu2022fair} have focused on reducing the global model's decision bias towards various protected groups, such as gender, race, or age. However, due to the heterogeneity of client data, these efforts cannot guarantee an equitable distribution of decision bias across clients.
Recently, there has been an increasing interest in promoting fairness among clients in FL. Li et al. \cite{li2021ditto}  have introduced Ditto, allowing clients to fine-tune the global model using their local data. The primary objective of Ditto is to optimize local performance rather than to mitigate performance disparities.
To mitigate performance disparities across clients, 
Mohri et al. \cite{mohri2019agnostic}  have proposed agnostic federated learning (AFL), a min-max multi-objective optimization that improves accuracy for the worst-performing client.  Cui et al. \cite{cui2021addressing}  have proposed fair and consistent federated learning (FCFL), a multi-objective optimization designed to simultaneously maximize performance and ensure accuracy consistency across different local clients. The consistency in accuracy is also achieved by optimizing for the worst performing client, reducing only the gap between the worst and the best clients, and it does not satisfy the requirement for equitable performance to mitigate the Matthew effect, which demands minimized variance in performance distribution across clients. Furthermore, the disparity in decision bias among clients has not been considered.
Li et al. \cite{li2019fair}  have proposed q-FFL, a heuristic method designed to achieve uniform accuracy across clients by adjusting the weights during aggregation to amplify the impact of clients with resource disadvantages. However, this approach does not guarantee that the global model achieves Pareto optimality. 
Pan et al. \cite{pan2023fedmdfg} have proposed FedMDFG, which incorporates cosine similarity between the loss vectors of clients and the unit vector as a fairness objective in local loss functions to achieve equitable accuracy among clients. However, this work does not address decision bias. Given the inherent trade-off between accuracy and decision bias for each client, as identified by Wang et al. \cite{wang2021understanding}, it is also necessary to consider the impact on decision bias when adjusting the accuracy distribution across clients. 
As simultaneously achieving both optimality and equity in accuracy and decision bias across clients is still an unexplored area and is essential for alleviating the impact of the Matthew effect in welfare-oriented collaborative training,  we propose anti-Matthew effect federated learning (\textit{anti-Matthew FL}) in this work to attain these objectives.

 \section{Preliminaries}
\subsection{Federated Learning}
We focus on horizontal FL \cite{yang2019federated}, which involves $N$  clients, each associated with a specific dataset $\mathcal{D}_k = \{X_k, A_k, Y_k\}$, where $k\in  \left \{ 1,...,N \right \} $, $X_k$ denotes the general attributes of the data without protected information,  $A_k$ denote a protected attribute, such as gender, race, or religion, and $Y_k$  denoted truth label.  The FL procedure involves multiple rounds of communication between the server and the clients.
In each round, the server sends the global model  $h_{\theta}$ with parameter $\theta$ to the clients, who then train their local models on their local private datasets $\left \{ \mathcal{D}_1,...,  \mathcal{D}_N\right \}$, resulting in local models $\left \{ h_{\theta_1} ,...,h_{\theta_N}\right \} $. The server then aggregates the local parameters and updates the global model for the next communication round \cite{mcmahan2017communication}. The original FL  \cite{mcmahan2017communication}  aims to minimize the average empirical risk loss over all the clients’ datasets, and the optimal hypothesis  parameter ${\theta^*}$ satisfies:
\begin{equation}
\label{eq:s_fl_goal}
\theta ^*=\arg\,\,\min_{\theta \in \Theta }\sum_{k=1}^{N} \left ( \frac{\left | \mathcal{D} _k \right | }{ {\textstyle \sum_{k=1}^{N}\left | \mathcal{D} _k \right |} }  l_k\left ( \hat{Y}_k,Y_k \right )\right ),  
\end{equation}
where  $\hat{Y}_k=h_\theta\left (X_k,{A}_k  \right ) $ is the output of the hypothesis  $h_\theta$ when input $\left (X_k,{A}_k  \right )$ and $l_k(\cdot)$ is the loss function for $k$-th client.

We instantiate the FL framework within the context of a binary classification problem, incorporating a binary sensitive attribute. For the optimization process, we employ the \textit{binary cross entropy loss} to instantiate the loss function $l_k$.
\begin{equation}
    \resizebox{0.9\hsize}{!}{$l_{k}(h) = -\frac{1}{\left | Y \right | }\sum_{i=1}^{\left | Y \right |}\left ( Y_k^i \log\left ( \hat{Y}_k^i \right )  + \left ( 1 - Y_k^i  \right ) \log\left ( 1 - \hat{Y}_k^i \right )  \right ).$}
\end{equation}

 The decision bias refers to the statistical disparity observed in model outcomes across  different groups divided by protected attributes, such as gender, race, and region.
We employ two specific decision bias metrics, namely the \textit{True Positive Rate Parity Standard Deviation} (TPSD)  and the \textit{Accuracy Parity Standard Deviation} (APSD)  
 \cite{poulain2023improving}. 
The TPSD and APSD for the $k$-th  client are defined as:
\begin{equation}
\label{eq:bias_metric}
\begin{matrix}
        
    \text{TPSD:\,\,}f_{k}\left ( h\right )=\sqrt{\frac{\sum_{i=1}^{M}\left (\operatorname{Pr}\left (\hat{Y}_k=1 \mid A_k=i, Y_k=1\right )-\mu\right )^{2}}{M}},\\
    \text{APSD:\,\,} f_k\left (h\right ) =\sqrt{\frac{\sum_{i=1}^{M}\left (\operatorname{Pr}\left (\hat{Y}_k=Y_k \mid A_k=i\right )-\mu\right )^{2}}{M}},
\end{matrix}
\end{equation}
where   $\mu$ is the average  \textit{True Positive Rate} (TPR) or average accuracy under all groups divided by the values of the protected attribute, and $M$ is the number of possible values for the sensitive attribute $A_k$. A hypothesis $h_\theta$ satisfies $\epsilon_b$-decision bias on $k$-th client if $f_k(h)\le \epsilon_b$, where $\epsilon_b$ is the predefined budget for the decision bias. 

\subsection{Anti-Matthew Fairness}
The \textit{anti-Matthew fairness}  refers to the model providing equal performance across clients. 
Pan et al. \cite{pan2023fedmdfg} propose \textit{cosine similarity}    between local losses and a unit vector $p=\mathbf{1}$ to assess the equity of model losses across all clients. This metric, however, lacks granularity in distinguishing each client's performance and fails to impose precise performance constraints. 
To address this limitation, we propose an alternative approach for evaluating model performance equality across clients. We measure this equality through the absolute deviation of each client's performance from the mean performance of all clients. Specifically, a hypothesis $h$  satisfies $\epsilon_{vl}$-\textit{anti-Matthew fairness} concerning accuracy performance and $\epsilon_{vb}$-\textit{anti-Matthew fairness} concerning decision bias performance if:
\begin{equation}
\begin{matrix}
\left | l_k(h)-\bar { l } (h)   \right | \le\epsilon _{vl},\left | f_k(h)-\bar { f } (h)    \right | \le\epsilon _{vb}, k \in \left \{ 1,...,N \right \} ,
\end{matrix}
\end{equation}
where $\bar { l } (h) =\frac{1}{N} {\textstyle \sum_{k=1}^{N}l_k(h)}  $ and $\bar { f } (h) =\frac{1}{N} {\textstyle \sum_{k=1}^{N}f_k(h)} $ are the average empirical risk loss and average decision bias, respectively, and $\epsilon _{vl}$ and $\epsilon _{vb}$ are the predefined budgets for the \textit{anti-Matthew fairness} on accuracy and decision bias, respectively.
\subsection{Anti-Matthew  Federated Learning}
To  achieve a global model that provides both high and \textit{anti-Matthew fairness} across clients, we propose a novel framework called \textit{anti-Matthew  effect federated learning (anti-Matthew FL)}, in which the training goals can be formulated as a multi-constrained multi-objectives optimization (MCMOO) problem. 
\begin{definition} (Anti-Matthew  FL) We formalize the \textit{anti-Matthew  FL} training objective as follows:
    \begin{equation}
\label{eq:ori_goal}
\begin{matrix}
\underset{h\in \mathcal{H}^* }{\min} \left \{  l_1\left ( h \right ),...,  l_N\left ( h \right )   \right \} , \text{s.t.}  \left \{  f_k\left ( h \right )    \right \} _{k=1}^N \le \epsilon _b,\\
  \left \{\left | l_k\left ( h  \right ) -\bar { l } (h)   \right | \right \} _{k=1}^N\le \epsilon _{vl},  \left \{\left | f_k\left ( h  \right ) -\bar { f } (h)  \right | \right \} _{k=1}^N\le \epsilon _{vb},
\end{matrix}
\end{equation}
where $h$ is a hypothesis from a hypothesis set $\mathcal{H}^*$. 
\end{definition}
The MCMOO problem seeks to minimize the empirical risk losses for all clients while ensuring each client has a $\epsilon_b$-decision bias. It also satisfies  $\epsilon_{vl}$-\textit{anti-Matthew fairness} for accuracy and $\epsilon_{vb}$-\textit{anti-Matthew fairness} for decision bias. Finding the optimal solution to the MCMOO problem is nontrivial, as the objectives may conflict. Therefore, we aim to identify the Pareto-optimal hypothesis
$h$,  which is not dominated by any other $h' \in \mathcal{H}$. The definitions of Pareto optimal and Pareto front \cite{lin2019pareto} are as follows:
 
\begin{definition} \text{(Pareto Optimal and Pareto Front)} In a multi-objective optimization problem with loss function $l(h)=\left \{l_1(h),...,l_N(h)  \right \} $, we say that for $h_1,h_2 \in \mathcal{H} $,  $h_1$ is dominated by $h_2$ if $\forall i\in \left [ N \right ], l_i(h_2) \le l_i(h_1)$ and $\exists i \in \left [ N \right ], l_i(h_2) < l_i(h_1)$. A solution $h$ is   Pareto optimal if it is not dominated by any other $h' \in \mathcal{H}$. The collection of  Pareto optimal solutions is called the  Pareto set. The image of the Pareto set in the loss function space is called the Pareto front.
\end{definition}
\section{ Detailed  Mechanism of Anti-Matthew FL  }
\subsection{ Stages to Obtain Pareto Optimal}
Fig. \ref{fig:stages} illustrates the feasible decision space of Eq. \eqref{eq:ori_goal}, which is bounded by the intersection of two hypothesis sets:  $\mathcal{H} _{B}$, and $\mathcal{H} _{E}$.
The  $\mathcal{H}_B$  contains  hypotheses satisfying $\epsilon_b$-decision bias in  each client, 
     \begin{equation}
         \left \{  f_k\left ( h \right )    \right \} _{k=1}^N \le \epsilon _b,\forall h\in \mathcal{H} _B.
     \end{equation}
The   $\mathcal{H}_{E}$ contains hypotheses  that satisfy $\epsilon_{vl}$-\textit{anti-Matthew fairness} on accuracy and $\epsilon_{vb}$-\textit{anti-Matthew fairness} on decision bias across all clients,
    \begin{equation}
\begin{matrix}
 \left \{ \left | l_k\left ( h  \right ) - \bar{l} \left ( h \right )    \right |  \right \} _{k=1}^N\le \epsilon _{vl}, \\ \left \{\left | f_k\left ( h  \right ) - \bar{f} \left ( h \right )    \right |  \right \} _{k=1}^N\le \epsilon _{vb}, \forall h\in \mathcal{H} _{E}.
\end{matrix}
    \end{equation}

The set $\mathcal{H}^* \subset \mathcal{H}_{B} \cap \mathcal{H}_{E}$ represents the Pareto set corresponding to Eq. \eqref{eq:ori_goal}. Formally, for any hypotheses $h$ in $\mathcal{H}^*$ and $h'$ in $\mathcal{H}_{B} \cap \mathcal{H}_{E}$, it holds that $h'\not\prec h$, where $\prec$ denotes the Pareto dominance relation.

\begin{proposition}
    (Existence of Feasible Hypothesis)
    The feasible decision space, $\mathcal{H}_B \cap \mathcal{H}_E \neq \emptyset$, is guaranteed to be non-empty. 
\end{proposition}

\begin{proof}
  For $h\in \mathcal{H}_E$, if $\bar{f}(h) \leq \epsilon_b - \epsilon_{vb}$, then $\left \{\left | f_k\left ( h \right ) - \bar{f} \left ( h \right ) \right | \right \}_ {k=1}^N\le \epsilon_{vb}$, this implies that $h$ also satisfies $\left \{ f_k\left ( h \right ) \right \} _{k=1}^N \le \epsilon_b$ and belongs to $\mathcal{H}_B$. Therefore, $\mathcal{H}_B \cap \mathcal{H}_E \neq \emptyset$ and includes at least $h$ that fulfills the requirement $\bar{f}(h) \leq \epsilon_b - \epsilon_{vb}$.
\end{proof}
 
\begin{figure}[t]
\centering
\includegraphics[width=8.5cm]{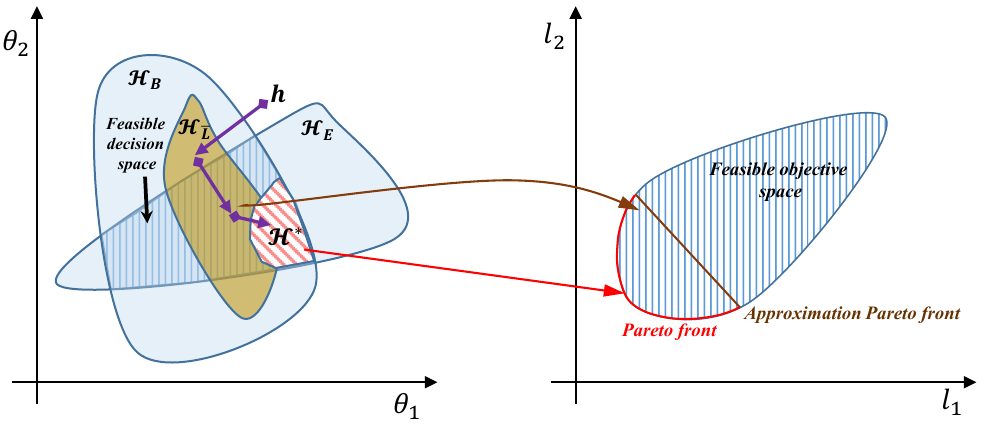}
\caption{ Optimization paths to achieve a Pareto solution $h \in \mathcal{H}^*$ in \textit{anti-Matthew FL}.}
\label{fig:stages}
\end{figure}

Finding the Pareto set for \textit{anti-Matthew FL} poses significant challenges due to the highly constrained nature of the feasible decision space. Additionally, when dealing with a large number of objectives, the optimization of one objective may have adverse effects on others. 
To address this issue, we employ a linear scalarization technique to construct an approximate Pareto front. Average weights are assigned to each objective, transforming the multi-objectives into a single surrogate objective. This surrogate objective forms the convex segment of the Pareto front, as illustrated in Fig. \ref{fig:stages}, which is denoted as $\mathcal{H}_{\bar L}$. The hypothesis in the $\mathcal{H}_{\bar L}$ satisfies:
  \begin{equation}
\bar  l\left ( h \right ) \le  \bar  l\left ( h' \right ),\forall h \in \mathcal{H} _{\bar L},h'\notin   \mathcal{H} _{\bar L} .
\end{equation}
Compared to $\mathcal{H}^{*}$, $\mathcal{H}_{\bar L}$ is easier to obtain and  can serve as an intermediate set, from which we propose a  three stages optimization path: $h^{0} \rightarrow \mathcal{H}_{B}\cap\mathcal{H}_{\bar L} \rightarrow \mathcal{H}_{B}\cap\mathcal{H}_{E}\cap\mathcal{H}_{\bar L}\rightarrow \mathcal{H}^{*}$ (purple arrows in Fig. \ref{fig:stages}), and decompose the \textit{anti-Matthew FL} into three sub-problems as follows:

\noindent\textbf{Stage 1: Constrained Minimization Problem.}
We define a constrained minimization  problem on the hypothesis set $\mathcal H$ to obtain  a hypothesis  $h'\in \mathcal{H}_{B} \cap \mathcal{H}_{\bar L} $,
\begin{equation}
\label{eq:stage1}
\underset{h\in \mathcal{H} }{\min} \,\,\bar   l\left ( h \right ), \text{s.t.} \left \{  f_k\left ( h \right )    \right \} _{k=1}^N \le \epsilon _b.
\end{equation}
By solving Eq. \eqref{eq:stage1}, we obtain  $h'$ that 1) satisfies $\epsilon _b$-decision bias for each client and 2) minimizes the average empirical risk loss among all clients.

\noindent\textbf{Stage 2:  Multi-Constrained Optimization Problem.}
We formulate a multi-constrained  optimization problem to obtain a hypothesis   $h''\in \mathcal{H}_{B} \cap \mathcal{H}_{E} \cap \mathcal{H}_{\bar L} $,
\begin{equation}
\label{eq:stage2}                                                   
\begin{matrix}
\underset{h\in \mathcal{H}}{\min}\,\,  \bar l \left ( h\right ) ,  \text{s.t.} \left \{ \left | l_k\left ( h  \right ) -\bar l\left ( h \right ) \right |    \right \} _{k=1}^N\le  \epsilon_{vl} ,\\ \left \{\left | f_k\left ( h \right ) -\bar f\left ( h\right )    \right | \right \} _{k=1}^N\le \epsilon _{vb}, \left \{  f_k\left ( h \right )    \right \} _{k=1}^N \le \epsilon _b.   
\end{matrix}
\end{equation}
By solving Eq. \eqref{eq:stage2}, we obtain  $h''$ that, compared to $h'$, exhibits the following properties:  1)  it provides $\epsilon_{vl}$-\textit{anti-Matthew fairness} on accuracy; and 2) it provides $\epsilon_{vb}$-\textit{anti-Matthew fairness} on decision bias. 

\noindent\textbf{Stage 3: Multi-Constrained Pareto Optimization Problem.}
Solely prioritizing the minimization of the weighted sum $\bar{l}(h)$ during optimization may harm the performance of individual clients. To mitigate this concern, we formulate a multi-constrained Pareto optimization problem aimed at refining $h''$ to $h^* \in \mathcal{H}^{*}$. In this process, the empirical risk loss of each client is further reduced until Pareto optimality is attained. After this stage, the loss of each client cannot be further minimized without causing a negative impact on the loss of other clients.
\begin{equation}
\label{eq:stage3}
    \begin{matrix}
\underset{h\in \mathcal H }{\min}\,\, \left \{   l_1\left ( h \right ) ,...,  {l} _N\left ( h \right )\right \}  ,\\ \text{s.t.} \left \{  f_k\left ( h\right )    \right \} _{k=1}^N \le \epsilon _b,\left \{\left | l_k\left ( h  \right ) -\bar l\left ( h \right )    \right | \right \} _{k=1}^N\le \epsilon _{vl},
\\ \left \{\left | f_k\left ( h  \right ) -\bar f\left ( h \right )    \right | \right \} _{k=1}^N\le \epsilon _{vb}, \bar{l} \left ( h \right ) \le \bar{l} \left ( h'' \right )  .
\end{matrix}
\end{equation}

\subsection{Three-Stage Gradient Descent to Obtain $\mathcal{H}^*$}
\label{sec:gba}
    To obtain the convergent solution for the sub-problems defined in Eq. \eqref{eq:stage1} to Eq. \eqref{eq:stage3}, we propose a three-stage multi-gradient descent  algorithm for acquiring $h^* \in \mathcal{H}^*$ that is well-suited for implementation within \textit{federated stochastic gradient descent} (FedSGD). Given a hypothesis $h_{\theta^t}$ parameterized by $\theta^t$, at iteration $t+1$, the update rule of the parameters is $\theta^{t+1} = \theta^{t} + \eta d$, where $d$ represents gradient descent direction, and $\eta$ is the step size. In the context of an optimization problem with $N$ objectives, i.e.,  $\min \left \{ l_1(h_{\theta}),...,l_N(h_{\theta}) \right \}$, the gradient descent direction $d$ is considered effective in guiding the optimization towards minimization if   $\left \{ d^{*T}\nabla_{\theta} l_i \left ( h_{\theta }\right )  \right \} _{i=1}^N\le 0$. 
As the gradient direction $d$ resides within the convex hull of the gradients of all objectives, denoted as $G=\left [\nabla _{\theta}  l_1(h_{\theta}),...,\nabla _{\theta}l_N(h_{\theta}) \right ] $ \cite{desideri2012multiple}, we can obtain the optimal gradient descent direction $d^*$  by performing a linear transformation on $G$   using an $N$-dimensional vector $\alpha^*$,
\begin{equation}
\resizebox{0.8\hsize}{!}{$\begin{matrix}
     d^*=\alpha^{*T}G,\,\,\mathrm{where} \,\,
\alpha^*=\arg\,\underset{\alpha}{\min} \sum_{i=1}^{N} \alpha_{i} \nabla_{\theta}  l_{i}\left ( h_{\theta} \right ),\\ \text {s.t. }\,\, \sum_{i=1}^{N} \alpha_{i}=1, \alpha_{i} \geq 0, \forall i \in \left [ N \right ].
\end{matrix}$}
\end{equation}
\textbf{Optimal Gradient Descent
Direction in Stage 1.} We initially convert Eq. \eqref{eq:stage1} into an equivalent single-constraint optimization problem by imposing a constraint solely on the maximum value, as follows:
\begin{equation}
\label{eq:stage1_mean}
\underset{h\in \mathcal{H} }{\min}\,\, \bar   l\left ( h \right ),\,\,\text{s.t.}\,\,\max \left \{  f_k\left ( h \right )  \right \} _{k=1}^{N}  \le \epsilon _b.
\end{equation}
Denoting the $\max \left \{  f_k\left ( h \right )  \right \} _{k=1}^{N}$ as $f_{max}(h)$, the  descent gradient  of Eq. \eqref{eq:stage1_mean} lies in the convex hull of $G'=\left [ \nabla_{\theta} \bar l(h) ,\nabla_{\theta} f_{max}(h) \right ] $.  Drawing inspiration from Cui et al.'s work \cite{cui2021addressing}, we adopt an alternating optimization strategy to alleviate computational burden: if the  $ \epsilon_b$-decision bias  is satisfied within the worst-case client, only $\bar{l}(h)$ is further optimized,
\begin{equation}
\label{eq:stage1_d_1}
    d^*=\arg\, \underset{d\in G'}{\min} \,\, d^T\nabla_{\theta} \bar l(h),\,\,\text{if}\,\,f_{max}(h) \leq \epsilon_b.
\end{equation}
Otherwise, we optimize towards a descent direction  $d$, which minimizes $f_{max}(h)$ while ensuring that $\bar{l}(h)$ does not increase, as follows:
\begin{equation}
\label{eq:stage1_d_2}
\begin{matrix}
        d^*=\arg \,\underset{d\in G'}{\min}\,\,  d^T\nabla_{\theta} f_{max}(h),\\ \text{s.t.}\,\,d^T\nabla_{\theta} \bar l(h)\le 0,\,\,\text{if}\,\,f_{max}(h) >\epsilon_b. 
\end{matrix}
\end{equation}
The gradient direction in Eq. \eqref{eq:stage1_d_1} and Eq. \eqref{eq:stage1_d_2} is optimized to minimize the loss while better satisfying the $\epsilon_b$-decision bias constraint. This optimization process results in a hypothesis $h'$ that achieves a balanced trade-off between accuracy and decision bias.

\noindent \textbf{Optimal Gradient Descent
Direction in Stage 2.} 
To reduce the computational complexity associated with handling $O(N)$ constraints in Eq. \eqref{eq:stage2}, we focus on optimizing \textit{anti-Matthew fairness} specifically for the worst-case client. Furthermore, to better achieve $\epsilon_{vl}$-\textit{anti-Matthew fairness} on accuracy and $\epsilon_{vb}$-\textit{anti-Matthew fairness} on decision bias, we modify Eq. \eqref{eq:stage2} by treating \textit{anti-Matthew fairness} as objectives and introducing a constraint $\bar{l}\left ( h\right ) \leq \bar{l}\left ( h'\right )$ to prevent the degradation of model accuracy performance. 

We optimize \textit{anti-Matthew fairness} of accuracy and decision bias alternately, i.e.,
if $\left \{\left | l_k\left ( h \right ) -\bar l\left ( h\right )    \right | \right \} _{k=1}^N\le \epsilon _{vl}$,

\begin{equation}
\label{eq:stage2_mean1} 
\begin{matrix}
\underset{h\in \mathcal{H}}{\min} \max  \left \{ \left | f_k\left ( h  \right ) -\bar f(h)   \right |-\epsilon _{vb} \right \}_{k=1}^{N}   ,\\ \text{s.t.}\,\,  \max\left \{  f_k(h) \right \}_{k=1}^{N}   \le \epsilon _b,\bar l\left ( h\right )\le \bar l\left ( h'\right ),
\end{matrix}
\end{equation}
else,
\begin{equation}
\label{eq:stage2_mean2} 
\begin{matrix}
\underset{h\in \mathcal{H}}{\min} \max  \left \{ \left | l_k\left ( h  \right ) -\bar l(h)   \right |-\epsilon _{vl} \right \}_{k=1}^{N}   ,\\\text{s.t.}\,\, \max  \left \{ \left | f_k\left ( h  \right ) -\bar f(h)   \right | \right \}_{k=1}^{N}  \le \epsilon _{vb},\\ \max\left \{  f_k(h) \right \}_{k=1}^{N}   \le \epsilon _b,\bar l\left ( h\right )\le \bar l\left ( h'\right ) .    
\end{matrix}
\end{equation}

 Denoting the  $\max  \left \{   \left | l_k\left ( h  \right ) -\bar l(h)   \right |-\epsilon _{vl}  \right \}_{k=1}^{N} $ and  $\max  \left \{ \left | f_k\left ( h  \right ) -\bar f(h)   \right | -\epsilon _{vb} \right \}_{k=1}^{N}$  as $\hat{l}_{max}(h) $ and $\hat{f}_{max}(h) $, respectively, the gradient descent direction of Eq. \eqref{eq:stage2_mean1} lies in the convex hull of $G''=\left [ \nabla_{\theta} \hat f_{max},\nabla_{\theta} f_{max} ,\nabla_{\theta}  \bar{l}   \right ]  $.  We obtain the optimal $d^*$ as follows: 
\begin{equation}
\label{eq:stage2_d1}
\begin{matrix}
d^*=\arg\underset{d\in G''}{\min}d^T\nabla_{\theta} \hat{f}_{max}(h)   ,\\ 
\text{s.t.} \,\,d^T\nabla_{\theta} \bar{l}(h)\le 0, d^T\nabla_{\theta} f_{max}(h)\le 0\,\,\text{ if}\,\,f_{max}(h)>\epsilon_{b} .
\end{matrix}
\end{equation}
The gradient descent direction   of Eq. \eqref{eq:stage2_mean2}  lies in the convex hull of $G''=\left [ \nabla_{\theta}  \hat{l}_{max} ,\nabla_{\theta} \hat f_{max},\nabla_{\theta} f_{max} ,\nabla_{\theta}  \bar{l}   \right ]  $.  We obtain the optimal $d^*$ as follows: 
\begin{equation}
\label{eq:stage2_d2}
    \begin{matrix}
d^*=\arg\underset{d\in G''}{\min}d^T\nabla_{\theta} \hat{l}_{max}(h)   ,
\text{s.t.} \,\,d^T\nabla_{\theta} \bar{l}(h)\le 0,
\\ d^T\nabla_{\theta} \hat{f}_{max}(h)\le 0\,\, \text{ if}\,\,\hat{f}_{max}(h)>\epsilon_{vb} , \\ d^T\nabla_{\theta} f_{max}(h)\le 0\,\,\text{ if}\,\,f_{max}(h)>\epsilon_{b}.
\end{matrix}
\end{equation}
The constraints are dynamically imposed based on whether the current hypothesis  $h$ satisfies $\epsilon_b$-decision bias and $\epsilon_{vb}$-\textit{anti-Matthew fairness} on the decision bias. The optimal gradient direction in Eq. \eqref{eq:stage2_d1} and Eq. \eqref{eq:stage2_d2} is optimized to enhance the equality of performance among clients without compromising the overall model performance. This optimization process results in a hypothesis $h''$ that strikes a balance between \textit{anti-Matthew fairness} and maximizing model performance.

\noindent \textbf{Optimal Gradient Descent
Direction in Stage 3.} To reduce the computational complexity of minimizing $ N $  objectives  and handling $O(N)$-constraints in  Eq. \eqref{eq:stage3}, we optimize the empirical risk loss for the worst-case client and impose constraints $ l_k\left ( h\right )\le  l_k\left ( h''\right ) $  to prevent the degradation of performance for other clients.

\begin{equation}
 \label{eq:stage3_mean}
 \begin{matrix}
 \underset{h\in \mathcal H }{\min} \max \left \{  l_1\left ( h \right ) ,...,  {l} _N\left ( h \right )\right \} ,\\ \text{s.t.}\,\,l_k(h)\le l_k(h''), \forall k \in [N],\\ \bar l(h)\le \bar l(h''),  \max  \left \{ f_k\left ( h  \right )   \right \}_{k=1}^{N} \le \epsilon _b,\\
 \max  \left \{ \left | l_k\left ( h  \right ) -\bar l(h)   \right | \right \}_{k=1}^{N} \le \epsilon _{vl},\\ \max  \left \{ \left | f_k\left ( h  \right ) -\bar f(h)   \right | \right \}_{k=1}^{N} \le \epsilon _{vb}.
\end{matrix}
\end{equation}
Denoting the $\max \left \{  l_k\left ( h \right )  \right \} _{k=1}^{N}$ as $l_{max}(h)$,  the gradient descent direction  lies in the convex hull of:
\begin{align*}
G^* = [&\nabla_\theta {l}_{max}(h), \nabla_\theta {l}_1(h),...,\nabla_\theta {l}_N(h), \nabla_\theta \bar {l}(h),\\
&\nabla_\theta f_{max}(h), \nabla_\theta \hat{l}_{max}(h), \nabla_\theta \hat{f}_{max}(h)].
\end{align*}
 We obtain the optimal $d^*$  as follows,

\begin{equation}
\label{eq:stage3_d}
    \resizebox{0.8\hsize}{!}{$\begin{matrix}
d^*=\arg\underset{d\in G^*}{\min}d^T \nabla_{\theta}  {l}_{max}(h) ,
\text{s.t.} \,\,d^T\nabla_{\theta}\bar{l}(h)\le 0,\\
d^T\nabla_{\theta} {f}_{max}(h)\le 0 \,\,\text{if}\,\,{f}_{max}(h)>\epsilon_{b} ,
\\ d^T\nabla_{\theta} \hat{l}_{max}(h)\le 0 \,\,\text{if}\,\,\hat{l}_{max}(h)>\epsilon_{vl} ,\\
 d^T\nabla_{\theta} \hat{f}_{max}(h)\le 0 \,\,\text{if}\,\,\hat{f}_{max}(h)>\epsilon_{vb},\\
 d^T\nabla_{\theta}l_i(h)\le 0, \forall i\in [N]\,\,and \,\,i\ne \arg \max\left \{l_k(h)  \right \} _{k=1}^{N} .
\end{matrix}$}
\end{equation}
 
The optimal gradient direction in Eq. \eqref{eq:stage3_d} is tailored to minimize the empirical risk loss on the worst client without adversely affecting the performance of other clients. This leads to further refinement from the approximation Pareto front, formed by linear scalarization, to a  Pareto optimal hypothesis $h^*$ of Eq. \eqref{eq:ori_goal}.

\subsection{Convergence and Time Complexity Analysis }

 \begin{proposition}(Convergence-Preserving Gradient Descent)
 For $N$ optimization objectives $\left \{ l_1(\theta^t),...,l_N(\theta^t) \right \} $ and the model parameter updating rule under  gradient descent direction $d$: $\theta^{t+1}=\theta^{t}+\eta d$; if $d^T \nabla_{\theta}l_i \leq 0$, there exists $\eta_0$ such that for $\forall \eta \in [0, \eta_0]$, the objectives $\left \{ l_1(\theta^t),...,l_N(\theta^t) \right \}$ will not increase.
 \label{prop: convergence}
    \end{proposition}
Based on Prop. \ref{prop: convergence}, and with each new stage constraining the gradient direction without compromising the performance achieved in the previous stage, the proposed three-stage gradient descent enables a convergent FL training process towards Pareto optimality.
 \begin{proposition}(Polynomial Time-Complexity)
In our proposed three-stage gradient descent,  solving  direction $d^*$ has a polynomial time complexity and does not exceed $\widetilde{O}^*\left(2^{2.38}\right)+\widetilde{O}^*\left(4^{2.38}\right)+\widetilde{O}^*\left((N+4)^{2.38}\right)$, where $N$ represents the number of clients.
 \label{prop: time}
    \end{proposition}

\section{Experiments}
\subsection{Settings}
\noindent \textbf{Datasets.} 
 \ding{182}  Synthetic dataset: we generate a synthetic dataset with a protected attribute: $A \sim Ber(0.5)$, two general attributes: $X_1\sim \mathcal{N}(0,1)$, $X_2\sim  \mathcal{N}(\mathbbm{1} (a>0),2)$. The label is setting by $Y\sim Ber(u^l\mathbbm{1}(x_1+x_2\le0)+u^h\mathbbm{1}(x_1+x_2>0) )$, where $\left \{ u^l,u_h \right \} =\left \{ 0.3,0.6 \right \} $ if $A=0$, otherwise,$\left \{ u^l,u^h \right \} =\left \{ 0.1,0.9 \right \} $.   We split the dataset into two clients based on whether $x_1\le -0.5$ to make the clients heterogeneous in distribution and size;   \ding{183} Adult  dataset \cite{uci_adult}: Adult is a binary classification dataset with more than $40000$ adult records for predicting whether the annual income is greater than $50K$. We split the dataset into two clients based on whether the individual's education level is a Ph.D and select $race$ as a protected attribute;  \ding{184} EICU dataset \cite{johnson2018generalizability}: the eICU dataset includes data records about clinical information and hospital details of patients who are admitted to ICUs.  We filter out the hospitals with data points less than $1500$, leaving $11$ hospitals in our experiments. Naturally, we treat each hospital as a  client and select $race$ as a protected attribute;  \ding{185} ACSPublicCoverage dataset \cite{NEURIPS2021_32e54441}: ACSPublicCoverage dataset is used to predict whether an individual is covered by public health insurance  in the United States.   The dataset is collected in the year 2022. Naturally, we treat each state  as a  client, generating $51$ clients in our experiments and select $gender$ as a protected attribute.


\noindent \textbf{Baselines.} 
 \ding{182} \textit{FedAvg} \cite{mcmahan2017communication}: the original FL algorithm for distributed training of private data. It does not consider fairness for different demographic groups and different clients;   \ding{183} \textit{FedAvg + FairBatch} \cite{roh2021fairbatch}: each client adopts the state-of-the-art FairBatch in-processing debiasing strategy  on its local training data and then aggregation uses FedAvg; \ding{184} \textit{FedAvg+FairReg}: a local processing method by  optimizing the linear scalarized objective with the fairness regularizations of all clients; \ding{185} \textit{Ditto} \cite{li2021ditto}: an FL framework to achieve training fairness by allowing clients to  fine-tune the received global model on the local data;  \ding{186} \textit{q-FFL} \cite{li2019fair}: an FL framework to achieve  accuracy consistency among clients by weighing different local clients differently;  \ding{187} \textit{FCFL} \cite{cui2021addressing}: an FL framework to achieve  accuracy consistency across different local clients meanwhile to minimize decision bias; \ding{188} \textit{FedMDFG} \cite{pan2023fedmdfg}: an FL framework to achieve accuracy consistency across different local clients.

\noindent \textbf{Hyperparameters.}
We divide the communication rounds into three stages, each with $750$, $750$, and $500$ rounds, respectively, to ensure that the global model is fully updated and converges in each stage. In the constraint budgets setting, we set the decision bias budget $\epsilon_b$, the \textit{anti-Matthew fairness} budget on accuracy $\epsilon_{vl}$, and the \textit{anti-Matthew fairness} budget on decision bias $\epsilon_{vb}$ to half of the related-performance achieved by the original \textit{FedAvg}. For example, as shown in Tab. \ref{Tab:performance}, on the synthetic dataset experiments, the Avg. local TPSD of \textit{FedAvg} is $0.2480$, so we set $\epsilon_b=0.1$. The Std. local accuracy  of \textit{FedAvg} is $0.0283$, so we set $\epsilon_{vl}=0.01$. The Std. local TPSD  of \textit{FedAvg} is $0.0819$, so we set $\epsilon_{vb}=0.04$. We use the same parameter-setting strategy for other datasets. Since the constraints may conflict with each other, this setting allows us to better evaluate the superior performance of our proposed \textit{anti-Matthew FL} and avoid making a constraint too tight, which may result in a solution that is only optimal on that constraint.

\begin{table}[t]
\caption{Local performance under TPSD metric.}
\label{Tab:performance}
\centering
 \renewcommand{\arraystretch}{1.3}
\tiny
\begin{threeparttable}
\begin{tabular}{@{}cccccc@{}}
\toprule
\multirow{3}{*}{Dataset}   & \multirow{3}{*}{Method} & 

\multicolumn{4}{c}{Model Performance}                           \\ 
\cmidrule(l){3-6} 
                           &                         & \multicolumn{2}{c}{Local Acc.} & \multicolumn{2}{c}{Local TPSD} \\ \cmidrule(l){3-6} 
                           &                         & Avg.    & Std.($\epsilon_{vl}$)    & Avg.($\epsilon_{b}$)    & Std.($\epsilon_{vb}$)     \\ \midrule
\multirow{8}{*}{\begin{tabular}[c]{@{}l@{}}Synthetic\\ $\epsilon_{b}=0.1$\\ $\epsilon_{vl}=0.01$\\ $\epsilon_{vb}=0.04$\end{tabular}} & FedAvg                  & \textbf{.7735}   & .0283$(\times  )$        & .2480$(\times  )$   & .0819$(\times  )$     \\
                           & q-FFL                  & \textbf{.7735}  & .0283$(\times  )$       & .2480$(\times  )$  & .0819$(\times  )$    \\
                           & Ditto                   & .7229   & .0132$(\times  )$      & .2703$(\times  )$   & .0566$(\times  )$     \\
                           & FedMDFG                 & .7717   & \textbf{.0068}$(\surd )$         & .2473$(\times  )$  & .0662$(\times  )$       \\ 
                           & FedAvg+FairBatch        &  \cellcolor[HTML]{FFCCC9} .6360   & 
 \cellcolor[HTML]{FFCCC9} .0643$(\times  )$        &  \cellcolor[HTML]{FFCCC9} .1040$(\approx  )$  & \cellcolor[HTML]{FFCCC9} .0798$(\times  )$       \\
                           & FedAvg+FairReg          &  \cellcolor[HTML]{FFCCC9} .6227   & \cellcolor[HTML]{FFCCC9} .0394$(\times  )$         &   \cellcolor[HTML]{FFCCC9} .0952$(\surd )$  & \cellcolor[HTML]{FFCCC9} .0463$(\times  )$         \\
                           & FCFL                    & \cellcolor[HTML]{FFCCC9} .6330   & \cellcolor[HTML]{FFCCC9} .0177$(\times  )$         & \cellcolor[HTML]{FFCCC9} .0812$(\surd )$  & \cellcolor[HTML]{FFCCC9} .0435$(\approx  )$      \\
                           & Anti-Matthew FL (Ours)                  &  \cellcolor[HTML]{FFCCC9}  .6327  &\cellcolor[HTML]{FFCCC9}  .0087$(\surd )$          &   \cellcolor[HTML]{FFCCC9} \textbf{.0801}$(\surd )$   & \cellcolor[HTML]{FFCCC9} \textbf{.0359}$(\surd )$      \\ \midrule
\multirow{8}{*}{\begin{tabular}[c]{@{}l@{}}Adult  \\ $\epsilon_{b}=0.01$\\ $\epsilon_{vl}=0.03$\\ $\epsilon_{vb}=0.005$\end{tabular}}   & FedAvg                  &  \textbf{.7767}   & .0592$(\times  )$       & .0328$(\times  )$   & .0093 $(\times  )$        \\
                           & q-FFL                   & .7662   & .0400$(\times  )$        & .0472$(\times  )$   & .0046$(\surd  )$       \\
                           & Ditto                   & .7210   &  \textbf{.0039}$(\surd  )$        & .0169$(\times  )$   & .0111$(\times  )$       \\
                           & FedMDFG                 & .7656   & .0397$(\times  )$        & .0455$(\times  )$   & .0069$(\times   )$      \\
                           & FedAvg+FairBatch        &\cellcolor[HTML]{FFCCC9} .7756   & \cellcolor[HTML]{FFCCC9}.0556$(\times  )$        & \cellcolor[HTML]{FFCCC9}.0136$(\times  )$   & \cellcolor[HTML]{FFCCC9} .0128$(\times  )$       \\
                           & FedAvg+FairReg          & \cellcolor[HTML]{FFCCC9} .7663   & \cellcolor[HTML]{FFCCC9} .0686$(\times  )$        & \cellcolor[HTML]{FFCCC9} .0089$(\surd )$  & \cellcolor[HTML]{FFCCC9} .0066$(\times   )$        \\
                           & FCFL                    & \cellcolor[HTML]{FFCCC9} .7638   & \cellcolor[HTML]{FFCCC9} .0487$(\times  )$        & \cellcolor[HTML]{FFCCC9} .0143$(\times  )$   & \cellcolor[HTML]{FFCCC9} .0159$(\times  )$       \\
                           &  Anti-Matthew FL (Ours)                   & \cellcolor[HTML]{FFCCC9} .7685   & \cellcolor[HTML]{FFCCC9} .0281$(\surd  )$         &  \cellcolor[HTML]{FFCCC9} \textbf{.0036}$(\surd  )$   &  \cellcolor[HTML]{FFCCC9} \textbf{.0009}$(\surd  )$        \\ \midrule
\multirow{8}{*}{\begin{tabular}[c]{@{}l@{}}eICU\\ $\epsilon_{b}=0.02$\\ $\epsilon_{vl}=0.02$\\ $\epsilon_{vb}=0.02$\end{tabular}}     & FedAvg                  & .6560   & .0427$(\times  )$       & .0371$(\times  )$   & .0409$(\times  )$        \\
                           & q-FFL                   & \textbf{.6565}   & .0425$(\times  )$         & .0371$(\times  )$   & .0405$(\times  )$       \\
                           & Ditto                   & .6311   & .0216$(\approx  )$       & .0472$(\times  )$   & .0447$(\times  )$         \\
                           & FedMDFG                 & .6479   & .0227$(\times  )$        & .0311$(\times  )$   & .0266$(\times  )$       \\
                           & FedAvg+FairBatch        & \cellcolor[HTML]{FFCCC9} .6441   & \cellcolor[HTML]{FFCCC9} .0413$(\times  )$       & \cellcolor[HTML]{FFCCC9} .0304$(\times  )$   & \cellcolor[HTML]{FFCCC9} .0298$(\times  )$        \\
                           & FedAvg+FairReg          & \cellcolor[HTML]{FFCCC9} .6455   & \cellcolor[HTML]{FFCCC9} .0408$(\times  )$        &  \cellcolor[HTML]{FFCCC9} .0322$(\times  )$   & \cellcolor[HTML]{FFCCC9} .0266$(\times  )$        \\
                           & FCFL                    & \cellcolor[HTML]{FFCCC9} .6550   & \cellcolor[HTML]{FFCCC9} .0272$(\times  )$         & \cellcolor[HTML]{FFCCC9} .0344$(\times  )$   & \cellcolor[HTML]{FFCCC9} .0246$(\times  )$         \\
                           &  Anti-Matthew FL (Ours)                   &  \cellcolor[HTML]{FFCCC9} .6530   & \cellcolor[HTML]{FFCCC9} \textbf{.0195}$(\surd )$        & \cellcolor[HTML]{FFCCC9} \textbf{.0209}$(\approx  )$ & \cellcolor[HTML]{FFCCC9} \textbf{.0201}$(\approx  )$   
                           \\ \midrule
\multirow{8}{*}{\begin{tabular}[c]{@{}l@{}}ACSPublicCoverage\\ $\epsilon_{b}=0.015$\\ $\epsilon_{vl}=0.04$\\ $\epsilon_{vb}=0.015$\end{tabular}}     & FedAvg                  & \textbf{.7015}   & .0815$(\times  )$       & .0307$(\times  )$   & .0242$(\times  )$        \\
                           & q-FFL                   & .5888   & .0406$(\approx  )$         & .0236$(\times  )$   & .0202$(\times  )$       \\
                           & Ditto                   & .6578   & .0584$(\times  )$       & .0260$(\times  )$   & .0328$(\times  )$         \\
                           & FedMDFG                 & .5978   & .0453$(\times  )$        & .0215$(\times  )$   & .0187$(\times  )$       \\
                           & FedAvg+FairBatch        & \cellcolor[HTML]{FFCCC9}  .5972   & \cellcolor[HTML]{FFCCC9} 
 .0454$(\times  )$       & \cellcolor[HTML]{FFCCC9} .0214$(\times  )$   & \cellcolor[HTML]{FFCCC9} 
 .0188$(\times  )$        \\
                           & FedAvg+FairReg          & \cellcolor[HTML]{FFCCC9}  .5964   & \cellcolor[HTML]{FFCCC9} .0453$(\times  )$        &  \cellcolor[HTML]{FFCCC9} 
 .0202$(\times  )$   & \cellcolor[HTML]{FFCCC9}  .0194$(\times  )$        \\
                           & FCFL                    & \cellcolor[HTML]{FFCCC9}  .6285   & \cellcolor[HTML]{FFCCC9} 
 .0391$(\surd  )$         & \cellcolor[HTML]{FFCCC9}  .0215$(\times  )$   & \cellcolor[HTML]{FFCCC9} 
 .0187$(\times  )$         \\
                           &  Anti-Matthew FL (Ours)                  &  \cellcolor[HTML]{FFCCC9} .6237   & \cellcolor[HTML]{FFCCC9} \textbf{.0384}$(\surd )$        & \cellcolor[HTML]{FFCCC9} \textbf{.0147}$(\surd   )$ & \cellcolor[HTML]{FFCCC9} \textbf{.0128}$(\surd   )$   
                           
                           \\ \bottomrule
\end{tabular}
\begin{tablenotes} 
            \item  \colorbox[HTML]{FFCCC9}{Red}: Methods that consider reducing decision bias.
		\item \textbf{Bold}\,\,\,\,: Best performance compared to all algorithms.
            \item $(\times  ):$ Violation of constraint  exceeds 10\%.
            \item $(\approx ):$ Close to constraint, with violation of constraint not exceeding 10\%.
            \item $(\surd ):$ Satisfy constraint.
           
     \end{tablenotes} 
\end{threeparttable} 
\end{table}

\begin{table}[t]
\caption{Local performance under APSD metric.}
\label{Tab:performance_APSD}
 \renewcommand{\arraystretch}{1.3}
\centering
\tiny

\begin{threeparttable}
\begin{tabular}{@{}cccccc@{}}
\toprule
\multirow{3}{*}{Dataset}   & \multirow{3}{*}{Method} & \multicolumn{4}{c}{Model Performance}                           \\ \cmidrule(l){3-6} 
                           &                         & \multicolumn{2}{c}{Local Acc.} & \multicolumn{2}{c}{Local APSD} \\ \cmidrule(l){3-6} 
                           &                         & Avg.    & Std.($\epsilon_{vl}$)    & Avg.($\epsilon_{b}$)    & Std.($\epsilon_{vb}$)     \\ \midrule
\multirow{8}{*}{\begin{tabular}[c]{@{}l@{}}Synthetic\\ $\epsilon_{b}=0.08$\\ $\epsilon_{vl}=0.01$\\ $\epsilon_{vb}=0.04$\end{tabular}} & FedAvg                  & \textbf{.7735}  & .0283$(\times  )$        & .1663$(\times  )$   & .0791$(\times  )$    \\
                           & q-FFL                   & \textbf{.7735}  & .0283$(\times  )$        & .1663$(\times  )$   & .0791$(\times  )$     \\
                           & Ditto                   & .7229   & .0132$(\times  )$        & .1977$(\times  )$  & .0615$(\times  )$       \\
                           & FedMDFG                  & .7717   &.0068$(\surd  )$       & .1667$(\times  )$  & .0718$(\times  )$       \\  
                           & FedAvg+FairBatch        & \cellcolor[HTML]{FFCCC9} .6695   & \cellcolor[HTML]{FFCCC9} .0397$(\times  )$        & \cellcolor[HTML]{FFCCC9} .0999$(\times  )$   & \cellcolor[HTML]{FFCCC9} \textbf{.0401}$(\surd  )$       \\
                           & FedAvg+FairReg          &  \cellcolor[HTML]{FFCCC9} .6191   & \cellcolor[HTML]{FFCCC9} .0250$(\times  )$      &  \cellcolor[HTML]{FFCCC9} .0976$(\times  )$   & \cellcolor[HTML]{FFCCC9} .0720$(\times  )$      \\
                           & FCFL                    & \cellcolor[HTML]{FFCCC9} .6302   & \cellcolor[HTML]{FFCCC9} .0165$(\times  )$         &  \cellcolor[HTML]{FFCCC9} .0687$(\surd  )$   & \cellcolor[HTML]{FFCCC9} 0453$(\times  )$      \\
                           &  Anti-Matthew FL (Ours)                   &   \cellcolor[HTML]{FFCCC9} .6269  & \cellcolor[HTML]{FFCCC9} \textbf{.0029}$(\surd  )$         &   \cellcolor[HTML]{FFCCC9} \textbf{.0621}$(\surd  )$  & \cellcolor[HTML]{FFCCC9} .0430$(\approx  )$      \\ \midrule
\multirow{8}{*}{\begin{tabular}[c]{@{}l@{}}Adult  \\ $\epsilon_{b}=0.02$\\ $\epsilon_{vl}=0.03$\\ $\epsilon_{vb}=0.01$\end{tabular}}   & FedAvg                  & \textbf{.7767}   & .0592$(\times  )$        & .0494$(\times  )$   & .0257$(\times  )$    \\
                           & q-FFL                   & .7662   & .0400$(\times  )$        & .0386$(\times  )$   & .0089$(\surd  )$       \\
                           & Ditto                   & .7210   & \textbf{.0039}$(\surd  )$        & .0450 $(\times  )$  & .0481 $(\times  )$     \\
                           & FedMDFG                 & .7656   & .0397$(\times  )$        & .0436$(\times  )$   & .0068$(\surd  )$       \\ 
                           & FedAvg+FairBatch        & \cellcolor[HTML]{FFCCC9} .7726   & \cellcolor[HTML]{FFCCC9} .0556 $(\times  )$        & \cellcolor[HTML]{FFCCC9} .0218$(\approx  )$   & \cellcolor[HTML]{FFCCC9} .0076 $(\surd )$        \\
                           & FedAvg+FairReg          & \cellcolor[HTML]{FFCCC9} .7446   & \cellcolor[HTML]{FFCCC9} .0484$(\times  )$         & \cellcolor[HTML]{FFCCC9} .0109$(\surd  )$   & \cellcolor[HTML]{FFCCC9} .0186$(\times  )$       \\
                           & FCFL                    & \cellcolor[HTML]{FFCCC9} .7583   & \cellcolor[HTML]{FFCCC9} .0487$(\times  )$        & \cellcolor[HTML]{FFCCC9} .0109$(\surd  )$   & \cellcolor[HTML]{FFCCC9} .0195$(\times  )$       \\
                           & Anti-Matthew FL (Ours)                   & \cellcolor[HTML]{FFCCC9} .7549   & \cellcolor[HTML]{FFCCC9} .0284$(\surd  )$        &\cellcolor[HTML]{FFCCC9} \textbf{.0101}$(\surd  )$   & \cellcolor[HTML]{FFCCC9} \textbf{.0067}$(\surd  )$        \\ \midrule
\multirow{8}{*}{\begin{tabular}[c]{@{}l@{}}eICU\\ $\epsilon_{b}=0.02$\\ $\epsilon_{vl}=0.02$\\ $\epsilon_{vb}=0.02$\end{tabular}}     & FedAvg                  & .6560   & .0427$(\times  )$       & .0386$(\times  )$   & .0310$(\times  )$        \\
                           & q-FFL                   & \textbf{.6565}   & .0425$(\times  )$        & .0384$(\times  )$   & .0307$(\times  )$       \\
                           & Ditto                   & .6311   & .0216$(\approx  )$        & .0399$(\times  )$   & .0405$(\times  )$        \\ 
                           & FedMDFG                 & .6479   & .0227$(\times  )$         & .0395$(\times  )$    & .0301$(\times  )$        \\ 
                           & FedAvg+FairBatch        & \cellcolor[HTML]{FFCCC9} .6441   & \cellcolor[HTML]{FFCCC9} .0213$(\approx  )$        & \cellcolor[HTML]{FFCCC9} .0301$(\times  )$   & \cellcolor[HTML]{FFCCC9} .0227$(\times  )$       \\
                           & FedAvg+FairReg          & \cellcolor[HTML]{FFCCC9} .6346   & \cellcolor[HTML]{FFCCC9} .0765$(\times  )$        & \cellcolor[HTML]{FFCCC9} .0335$(\times  )$   & \cellcolor[HTML]{FFCCC9} .0267$(\times  )$         \\
                           & FCFL                    & \cellcolor[HTML]{FFCCC9} .6569   & \cellcolor[HTML]{FFCCC9} .0362$(\times  )$         & \cellcolor[HTML]{FFCCC9} .0331$(\times  )$   & \cellcolor[HTML]{FFCCC9} .0200$(\surd  )$      \\
                           &  Anti-Matthew FL (Ours)                   & \cellcolor[HTML]{FFCCC9} .6544   & \cellcolor[HTML]{FFCCC9} \textbf{.0192}$(\surd  )$          & \cellcolor[HTML]{FFCCC9} \textbf{.0207}$(\approx  )$     & \cellcolor[HTML]{FFCCC9} \textbf{.0182}$(\surd  )$         \\ \midrule
\multirow{8}{*}{\begin{tabular}[c]{@{}l@{}}ACSPublicCoverage\\ $\epsilon_{b}=0.01$\\ $\epsilon_{vl}=0.04$\\ $\epsilon_{vb}=0.01$\end{tabular}}     & FedAvg                  & \textbf{.7015}   & .0815$(\times  )$       & .0198$(\times  )$   & .0139$(\times  )$        \\
                           & q-FFL                   & .5888   & .0406$(\approx   )$         & .0160$(\times  )$   & .0140$(\times  )$       \\
                           & Ditto                   & .6578   & .0584$(\times  )$       & .0162$(\times  )$   & .0118$(\times  )$         \\
                           & FedMDFG                 & .5978   & .0453$(\times  )$        & .0158$(\times  )$   & .0123$(\times  )$       \\
                           & FedAvg+FairBatch        & \cellcolor[HTML]{FFCCC9} .5946   & \cellcolor[HTML]{FFCCC9} .0441$(\times  )$       & \cellcolor[HTML]{FFCCC9} .0152$(\times  )$   & \cellcolor[HTML]{FFCCC9} .0128$(\times  )$        \\
                           & FedAvg+FairReg          & \cellcolor[HTML]{FFCCC9} .5977   & \cellcolor[HTML]{FFCCC9} .0453$(\times  )$        &  \cellcolor[HTML]{FFCCC9} .0158$(\times  )$   & \cellcolor[HTML]{FFCCC9} .0122$(\times  )$        \\
                           & FCFL                    & \cellcolor[HTML]{FFCCC9} .6239   & \cellcolor[HTML]{FFCCC9} .0403$(\approx   )$         & \cellcolor[HTML]{FFCCC9} .0114$(\times  )$   & \cellcolor[HTML]{FFCCC9} .0095$(\surd  )$         \\
                           &  Anti-Matthew FL (Ours)                   &  \cellcolor[HTML]{FFCCC9} .6221   & \cellcolor[HTML]{FFCCC9} \textbf{.0402}$(\approx  )$        & \cellcolor[HTML]{FFCCC9} \textbf{.0106}$(\surd   )$ & \cellcolor[HTML]{FFCCC9} \textbf{.0080}$(\surd   )$    \\ \bottomrule
\end{tabular}
            
           
\end{threeparttable} 
\end{table}

\begin{figure}[t]
\centering
\includegraphics[width=8.5cm]{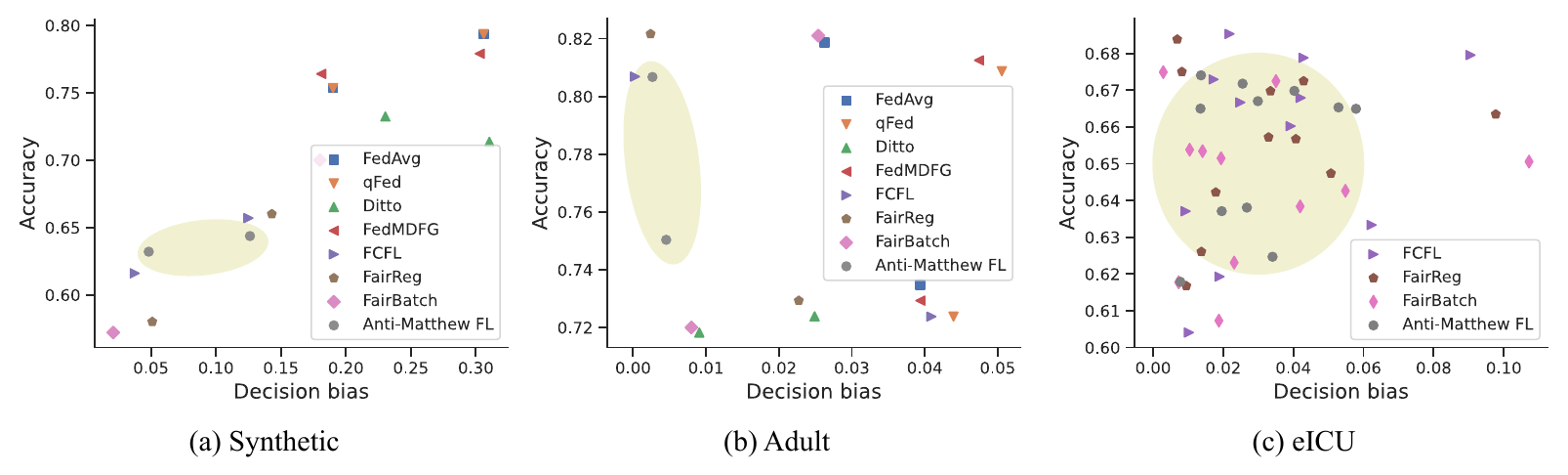}
\caption{Model performance distribution across different clients.}
\label{fig:all_client}
\end{figure}

\noindent \textbf{Evaluation Metrics.} To evaluate the effectiveness and equality of the global model’s performance across all clients, we introduce two evaluation metrics:
 \ding{182} Avg.: the average performance of the global model across all clients;  \ding{183} Std.: the standard deviation of the performance of the global model across clients. We utilize the TPSD  and APSD in Eq. \eqref{eq:bias_metric} as  metrics for decision bias.
 
\noindent \textbf{Normalization.}
We suggest normalizing the gradients $\mathbf{g}$ before computing the optimal gradient descent direction  $d^*$ when there is a large disparity among the gradients, as in the eICU and ACSPublicCoverage experiments where the numbers of clients are large,
\begin{equation}
     \mathbf{g} =\mathbf{g} /\sqrt{ {\textstyle \sum_{i=1}^{\left | g \right | }g_i^2} }.
\end{equation}

\subsection{Performance Comparison} 

We compare the global model performance of  \textit{anti-Matthew FL} with other SOTA baselines on four datasets. 
 Tab. \ref{Tab:performance} and \ref{Tab:performance_APSD} respectively illustrate the model performance under TPSD and APSD as decision bias metrics. It is evident that \textit{anti-Matthew FL}  achieves the best satisfaction of the fairness constraints. In terms of accuracy, the methods that do not consider decision bias (\textit{FedAvg}, \textit{q-FFL}, \textit{Ditto}, and \textit{FedMDFG}) have higher accuracy. However, as there is a trade-off between accuracy and decision bias, we compare our method with the baselines that also aim to reduce decision bias (\textit{FedAvg}+\textit{FairBatch}, \textit{FedAvg}+\textit{FairReg}, \textit{FCFL}). Under TPSD, \textit{anti-Matthew FL}  achieves the best constraint satisfaction with only  $0.3\%$ decrease in accuracy on the synthetic dataset,  $0.7\%$ decrease in accuracy on the adult dataset,  $0.2\%$ decrease in accuracy on the eICU dataset, and $0.7\%$ on the ACSPublicCoverage dataset.  Under APSD, \textit{anti-Matthew FL}  achieves the best constraint satisfaction with only  $0.3\%$ decrease in accuracy on the synthetic and  adult dataset,  $0.2\%$ decrease in accuracy on the eICU  and  ACSPublicCoverage dataset.   The results further validate the scalability of the \textit{anti-Matthew FL}, with clients ranging from $2$  to $51$.

  Fig. \ref{fig:all_client} illustrates the distribution of model performance across various clients.  We use TPSD  as decision
bias metric. For a clearer visualization, we prioritize the display of baselines considering decision bias (\textit{FedAvg}+\textit{Fairbatch}, \textit{FedAvg}+\textit{FairReg}, \textit{FCFL}) on the eICU dataset, which encompasses $11$ clients. The results demonstrate that \textit{anti-Matthew FL} ensures a more equitable performance distribution among clients, thereby indicating enhanced \textit{anti-Matthew fairness}.

\subsection{Convergence}
  Fig. \ref{fig:state} illustrates  the convergence of \textit{anti-Matthew FL}  within $2000$ communication rounds on  Adult and eICU datasets. We use TPSD  as decision
bias metric. The experiments are repeated $20$ times. We observe that the accuracy of the global model converges as the communication rounds increase, while the decision bias and the \textit{anti-Matthew fairness}  of accuracy and decision bias converge to  the predefined fairness budgets (indicated by the colored dashed lines in Fig. \ref{fig:state}). 
\begin{figure}[ht]
\centering
\includegraphics[width=8.5cm]{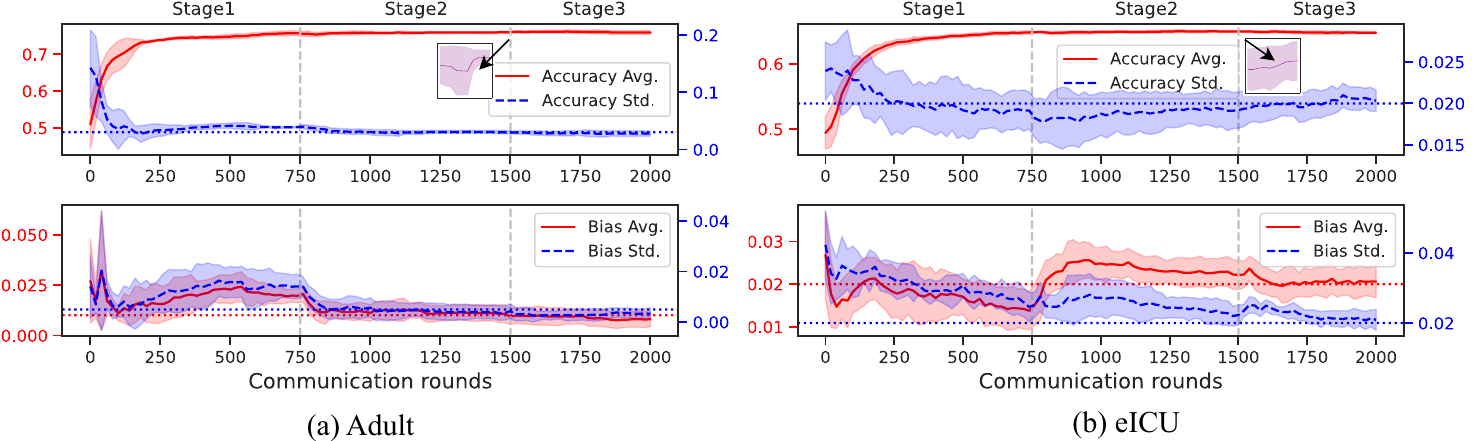}
\caption{Testing results during $2000$ communication rounds.}
\label{fig:state}
\end{figure}

\subsection{Hyperparameters Sensitivity}

In \textit{anti-Matthew FL}, we introduce three constraint budgets, namely, $\epsilon_{b}$, $\epsilon_{vl}$, and $\epsilon_{vb}$, to regulate the  decision bias and the  \textit{anti-Matthew fairness}. To investigate the impact of varied constraint configurations on the \textit{anti-Matthew FL} algorithm, 
we manipulated the values of $\epsilon_{b}$, $\epsilon_{vl}$, and $\epsilon_{vb}$ within the range of $\left [ 0.01, 0.02, 0.05, 0.1 \right ]$, respectively.

As shown in Fig. \ref{fig:param}, the decision bias of the global model is inversely proportional to $\epsilon_{b}$, which regulates the bias level. The $\epsilon_{vl}$ and $\epsilon_{vb}$ regulate the equality of the model performance across the clients, specifically, the smaller the  $\epsilon_{vl}$, the smaller the standard deviation of accuracy among the clients; the smaller the $\epsilon_{vb}$, the smaller the standard deviation of decision bias among the clients.

\begin{figure}[ht]
\centering
\includegraphics[width=8.5cm]{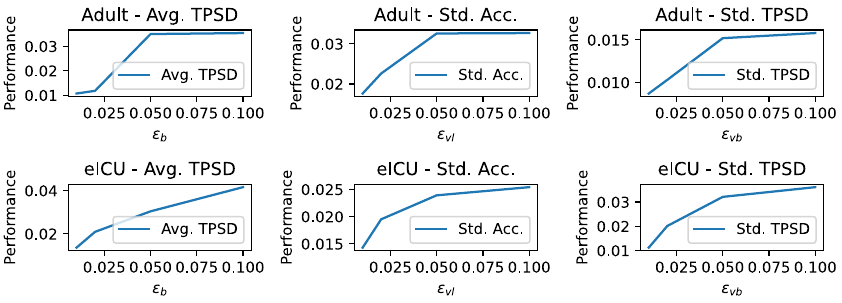}
\caption{ Effects of budgets $\epsilon_{b}$, $\epsilon_{vl}$ and $\epsilon_{vb}$ on the Avg. of local TPSDs, Std. of local accuracies and Std. of local TPSDs, respectively. }
\label{fig:param}
\end{figure}

\subsection{Ablation Study}
Compared to solving a MCMOO problem in a single-stage, a multi-stages method in \textit{anti-Matthew FL} offers the following advantages:
\begin{enumerate}
\item \textbf{Balanced Trade-offs:} The staged method promotes a more nuanced and balanced consideration of trade-offs;
\item \textbf{Partial Constraint Focus:} Dividing the problem into three distinct stages allows us to concentrate on satisfying partial constraints at each stage. This focused approach proves beneficial for navigating more effectively through a progressively stringent decision space in subsequent stages;
\item \textbf{Enhanced Convergence Speed:} The staged method contributes to a heightened convergence speed. In each stage, the emphasis is on solving a specific subproblem, thereby mitigating overall complexity and significantly improving the speed of arriving at solutions.
\end{enumerate}
In this experiment, we conducted an ablation study to systematically assess the indispensability of each stage within the three-stage algorithm.  The results presented in Tab. \ref{tab:as} demonstrate the realization of the desired hypothesis within the decision space constrained by fairness, achieved solely by integrating all three stages.

\begin{table}[ht]
     \caption{Ablation study of each stage in \textit{anti-Matthew FL}.}
     \label{tab:as}
\tiny
    \centering
    \renewcommand{\arraystretch}{1.7}
\begin{tabular}{cccc}
\hline
\textbf{Dataset}           & \textbf{Stage}          & \textbf{Avg.  Acc.$\pm $ Std.   Acc.} & \textbf{Avg.  TPSD$\pm $ Std.   TPSD} \\  \hline
\multirow{4}{*}{{\begin{tabular}[c]{@{}l@{}}Synthetic\\ $\epsilon_{b}=0.1$\\ $\epsilon_{vl}=0.01$\\ $\epsilon_{vb}=0.04$\end{tabular}} } & Stage 1                 & 0.6519$\pm$0.0169($\times$)   & 0.1033($\approx$)$\pm$0.0477($\times$)   \\ 
                           & Stage 2                 & 0.5740$\pm$0.0661($\times$)   & 0.0772 ($\surd$) $\pm$0.0740($\times$)   \\
                           & Stage 3                 & 0.6840$\pm$0.0983($\times$)   & 0.2483($\times$)$\pm$0.0670($\times$)   \\
                           & Stage 1+Stage 2+Stage 3 & 0.6327$\pm$ 0.0087($\surd$)   & 0.0801 ($\surd$)$\pm$0.0359($\surd$)   \\ \hline
\multirow{4}{*}{{\begin{tabular}[c]{@{}l@{}}Adult  \\ $\epsilon_{b}=0.01$\\ $\epsilon_{vl}=0.03$\\ $\epsilon_{vb}=0.005$\end{tabular}} }     & Stage 1                 & 0.7778$\pm$0.0530($\times$)   & 0.0167($\times$)$\pm$0.0212($\times$)   \\
                           & Stage 2                 & 0.7611$\pm$0.0450($\times$)   & 0.0069($\surd$)$\pm$0.0094($\times$)   \\
                           & Stage 3                 & 0.7383$\pm$0.0206($\surd$)   & 0.0109($\approx$)$\pm$0.0102($\times$)   \\
                           & Stage 1+Stage 2+Stage 3 & 0.7685$\pm$0.0281($\surd$)   & 0.0036($\surd$)$\pm$0.0009($\surd$)   \\ \hline

\multirow{4}{*}{{\begin{tabular}[c]{@{}l@{}}eICU\\ $\epsilon_{b}=0.02$\\ $\epsilon_{vl}=0.02$\\ $\epsilon_{vb}=0.02$\end{tabular}}  }      & Stage 1                 & 0.6488$\pm$0.0223($\times$)   & 0.0189($\surd$)$\pm$0.0263($\times$)   \\
                           & Stage 2                 & 0.6503$\pm$0.0207($\approx$)   & 0.0356($\times$)$\pm$0.0230($\times$)    \\
                           & Stage 3                 & 0.6446$\pm$0.0226($\times$)   & 0.0309($\times$)$\pm$0.0209($\approx$)   \\
                           & Stage 1+Stage 2+Stage 3 & 0.6530$\pm$0.0195($\surd$)   & 0.0209($\approx$)$\pm$0.0201($\approx$) \\ \hline \multirow{4}{*}{{\begin{tabular}[c]{@{}l@{}}ACSPublicCoverage\\ $\epsilon_{b}=0.015$\\ $\epsilon_{vl}=0.04$\\ $\epsilon_{vb}=0.015$\end{tabular}} }      & Stage 1                 & 0.6215$\pm$0.0526($\times$)   & 0.0194($\times$)$\pm$0.0160($\approx$)   \\
                           & Stage 2                 & 0.6441$\pm$0.0567($\times$)   & 0.0211($\times$)$\pm$0.0158($\approx$)    \\
                           & Stage 3                 & 0.6913$\pm$0.0754($\times$)   & 0.0071($\surd$) $\pm$0.0065($\surd$)   \\
                           & Stage 1+Stage 2+Stage 3 & 0.6237$\pm$0.0384 ($\surd$)   & 0.0147 ($\surd$)$\pm$0.0128 ($\surd$)  
                           \\ \hline
\end{tabular}
\end{table}
\subsection{Robustness}
We conduct robustness validation experiments on the eICU dataset and randomly select $4$ clients to be malicious. The malicious clients adopt the following attacks to dominate or disrupt the FL process: (1) Enlarge: the malicious clients enlarge the local gradients or local model parameters sent to the server to enhance their influence in the training process. In our experiments, we set the enlarging factor to $10$; (2) Random: the malicious clients send randomly generated local gradients or model parameters to the server to disrupt the training process; (3) Zero: The malicious clients send zero local gradients or zero model parameters to the server to disrupt the training process. 

Tab. \ref{tab: robust} shows the testing performance of the SOTA baselines and our proposed \textit{anti-Matthew FL} under attacks. The results demonstrate that our \textit{anti-Matthew FL} is robust and the performance of the global model does not degrade due to the presence of malicious clients, while other baselines suffer from different degrees of decline in testing performance. 

\begin{table}[ht]
     \caption{The test performance under attacks.}
\label{tab: robust}
\tiny
\centering
\renewcommand{\arraystretch}{1.7}
\begin{tabular}{@{}cccc@{}}
\toprule
\multirow{2}{*}{\textbf{Attack}} & \multicolumn{3}{c}{\textbf{Avg.  Acc.$\pm $ Std.   Acc./ Avg.   TPSD$\pm $ Std.   TPSD}} \\ \cmidrule(l){2-4} 
                                 & \textbf{Enlarge}                 & \textbf{Random}                & \textbf{Zero}                \\ \midrule
FedAvg           &.5307$\pm $.0414/.0205$\pm $.0669         &.5129$\pm $.0743/.0242$\pm $.0291        & .6225$\pm $.0503/.0245$\pm $.0411     \\ 
q-FFL            & .6278$\pm $\textbf{.0261}/.0242$\pm $.0250         &.5824$\pm $.0461/.0225$\pm $.0371        &.6492$\pm $.0386/.0216$\pm $.0379      \\
Ditto            & .6162$\pm $.0283/.0203$\pm $.0247        & .5130$\pm $.0771/.0273$\pm $.0331     &.6236$\pm $.0498/.0231$\pm $.0399      \\ 
FedMDFG          &.6512$\pm $.0345/.0217$\pm $.0340         &.6502$\pm $.0332/.0229$\pm $.0318        &.6466$\pm $.0271/.0213$\pm $.0211      \\ 
FedAvg+FairBatch & .6079$\pm $.0280/.0199$\pm $.0200       &.5095$\pm $.0724/.0297$\pm $.0335        &.6225$\pm $.0503/.0245$\pm $.0411      \\ 
FedAvg+FairReg   & .5643$\pm $.0582/.0272$\pm $.0372         & .6510$\pm $.0299/.0225$\pm $.0275        & .6511$\pm $.0292/.0228$\pm $.0283     \\ 
FCFL             &.6280$\pm $.0296/.0232$\pm $.0237         & \textbf{.6535}$\pm $.0263/.0225$\pm $.0233       &.6254$\pm $.0317/.0215$\pm $.0295      \\ 
Anti-Matthew FL              &\textbf{.6536}$\pm $.0278/\textbf{.0178}$\pm $\textbf{.0185}         &.6510$\pm $\textbf{.0244}/\textbf{.0203}$\pm $\textbf{.0158}        &\textbf{.6522}$\pm $\textbf{.0259}/\textbf{.0206}$\pm $\textbf{.0183}      \\ \bottomrule
\end{tabular}

\end{table}
\section{Conclusion}

To mitigate \textbf{Matthew effect} in federated learning (FL), which exacerbates resource disparities among clients and jeopardizes the sustainability of the FL systems, we defined a novel fairness concept, \textit{anti-Matthew fairness}. We analyzed the challenges of achieving \textit{anti-Matthew fairness} in the FL setting. Specifically, we formally defined \textit{anti-Matthew FL} as a multi-constrained, multi-objectives optimization problem. Subsequently, we designed an effective three-stage multi-gradient descent algorithm to achieve the Pareto optimal global model. We also provided theoretical analyses of the algorithm’s convergence properties. Finally, we conducted comprehensive empirical evaluations to demonstrate that our proposed \textit{anti-Matthew FL} outperforms other state-of-the-art baselines in achieving a high-performance global model with enhanced \textit{anti-Matthew fairness} among all clients.

\bibliography{ecai2024_conference}




\end{document}